\documentclass[conference]{IEEEtran}
\usepackage{subfigure}
\usepackage{graphicx}
\usepackage{amsmath}
\usepackage{bm}
\usepackage{url}
\usepackage{stmaryrd}
\usepackage{balance}
\usepackage{amssymb}
\usepackage{pgfplots}
\usepackage{pifont}
\usepackage{epsfig}
\usepackage{booktabs}
\usepackage{pgffor}
\usepackage{tikz}
\usepackage{multirow}
\usepackage{mathrsfs}
\usepackage{bbm}
\usepackage{helvet}
\usepackage{courier}
\usepackage{threeparttable}
\usepackage{algpseudocode} 
\usepackage{arydshln}
\usepackage{mathrsfs}
\usepackage{amsthm}
\usepackage[colorlinks,linkcolor=blue]{hyperref}
\makeatletter
\newcommand\xleftrightarrow[2][]{%
	\ext@arrow 9999{\longleftrightarrowfill@}{#1}{#2}}
\newcommand\longleftrightarrowfill@{%
	\arrowfill@\leftarrow\relbar\rightarrow}
\makeatother

\newcommand{\mb}{\mathbf}

\newtheorem{thm}{\textsc{Theorem}}
\newtheorem{defn}{\textsc{Definition}}

\newtheorem{prop}{Proposition}

\newcommand{\our}{\text{Ripple Walk}}
\newcommand{\ourt}{\text{Ripple Walk Training}}
\newcommand{\ourshort}{\text{RWT}}
\newcommand{\oursampler}{\text{Ripple Walk Sampler}}
\newcommand{\gcn}{\text{GCN}}
\newcommand{\gat}{\text{GAT}}
\newcommand{\gnn}{\text{GNNs}}
\newcommand{\adam}{\textsc{ADAM}}

\usepackage{algorithm}

\begin{document}

 


\title{Ripple Walk Training: A Subgraph-based Training Framework for Large and Deep Graph Neural Network}





\author{\IEEEauthorblockN{Jiyang Bai\textsuperscript{\textsection},
		Yuxiang Ren\textsuperscript{\textsection} and Jiawei Zhang}
	\IEEEauthorblockA{Department of Computer Science, Florida State University, FL, USA\\
		Email: bai@cs.fsu.edu,
		yuxiang@ifmlab.org,
		jiawei@ifmlab.org}}

\maketitle
\begingroup\renewcommand\thefootnote{\textsection}
\footnotetext{Equal contribution}
\endgroup

\begin{abstract}
	Graph neural networks (GNNs) have achieved outstanding performance in learning graph-structured data and various tasks. However, many current GNNs suffer from three common problems when facing large-size graphs or using a deeper structure: \textit{neighbors explosion}, \textit{node dependence}, and \textit{oversmoothing}. Such problems attribute to the data structures of the graph itself or the designing of the multi-layers {\gnn} framework, and can lead to low training efficiency and high space complexity.  To deal with these problems, in this paper, we propose a general subgraph-based training framework, namely \textbf{R}ipple \textbf{W}alk \textbf{T}raining ({\ourshort}), for deep and large graph neural networks. {\ourshort} samples subgraphs from the full graph to constitute a mini-batch, and the full GNN is updated based on the mini-batch gradient. 
	%
	We analyze the high-quality subgraphs to train GNNs in a theoretical way. 
	A novel sampling method {\oursampler} works for sampling these high-quality subgraphs to constitute the mini-batch, which considers both the randomness and connectivity of the graph-structured data. 
	Extensive experiments on different sizes of graphs demonstrate the effectiveness and efficiency of {\ourshort} in training various GNNs (GCN \& GAT). Our code is released in the https://github.com/anonymous2review/RippleWalk.
\end{abstract}

\section{Introduction}\label{sec:introduction}
Graph neural networks (GNNs) have achieved outstanding performance in graph-structured data based applications, such as knowledge graphs~\cite{WMWG17}, social medias~\cite{ren2020hgat}, and protein interface prediction~\cite{FBSB17}. GNNs learn nodes' high-level representations through a recursive neighborhood aggregation scheme~\cite{XHLJ10}. As the graph's scale increases and higher-order neighbors are considered, the recursive neighborhood aggregation can cause the number of neighbors to explode. We name this problem as \textit{neighbors explosion}. For each node, the \textit{neighbors explosion} will lead to the computation complexity exponentially increasing with the {\gnn} depth~\cite{chiang2019cluster} and the graph size. Therefore, some current works on GNNs (e.g., graph attention networks (GAT)~\cite{gat}) can only handle small-size graphs (normally less than 5000 nodes) with a shallow structure (less than 3 layers). Besides, the graph-structured data has the characteristics of  \textit{node dependence}, which means neighboring nodes affect each other in the learning process. As a result, the most current GNNs have to learn on the full graph, and when the size of the graph is too large, it is easy to reach the upper limit of the device memory. Even with a shallow {\gnn} structure, the memory space demand is extremely large. \textit{node dependence} also limits the performance of training methods such as mini-batch Stochastic Gradient Descent (SGD). Because calculating the loss of one node, {\gnn} require embeddings of all node neighbors, and its neighbors also need embeddings of their neighbors for aggregation. This increases the overhead of mini-batch SGD, especially for dense graphs and deeper GNNs. Another factor that limits the effectiveness of {\gnn} is \textit{oversmoothing}. Especially when {\gnn} go deeper and learn on the full graph, and it is unavoidable that node representations from different clusters mix up~\cite{zhao2019pairnorm}. But this aggregation is unexpected because nodes from different clusters do not meet the smoothness assumptions on the graph (close nodes are similar). Finally, the \textit{oversmoothing} will lead to node representations indistinguishable. Therefore, when the GNN has a deep structure, not only the training is more difficult, but \textit{oversmoothing} also impedes its performance.

To deal with the three problems mentioned above, some methods have emerged. GraphSAGE~\cite{hamilton2017inductive} learns a function that generates embeddings by sampling and aggregating features from a node's local neighborhood. FastGCN~\cite{chen2018fastgcn} utilizes Monte Carlo approaches to sample neighbors which avoids the \textit{neighbors explosion}. Chen et.al.~\cite{chen2017stochastic} develop control variate based algorithms that allow sampling an arbitrarily small neighbor size. They all use neighbor sampling to avoid \textit{neighbors explosion} and improve the training speed, but they can not handle the remaining problems. When the full graph's size is large, the memory overhead for learning on the full graph is unacceptable. These methods do not optimize the memory overhead when speed up training. Cluster-GCN~\cite{chiang2019cluster} has a training algorithm based on subgraphs, which are constructed by clustering on the full graph. The subgraphs are selected randomly to constitute mini-batches to train the GCN. However, the size of clusters in a graph is difficult to control. When very large subgraphs are constructed based on the clustering results, Cluster-GCN lacks scalability and cannot tackle the \textit{neighbors explosion}. Besides, the time and space overhead of clustering on a large graph are also nonnegligible.  

In this paper, we propose a general subgraph-based training framework, namely \textbf{R}ipple \textbf{W}alk \textbf{T}raining ({\ourshort}), for deep and large graph neural networks. {\ourshort} aims to handle all aforementioned problems simultaneously. {\ourshort} is developed from the mini-batch training, but there exist apparent differences. Instead of sampling neighbors and training on the full graph, {\ourshort} samples subgraphs from the full graph to constitute a mini-batch. The full GNN is updated based on the mini-batch gradient. We design a novel sampling method {\oursampler} for {\ourshort}, which considers both the randomness and connectivity of the graph-structured data to deal with those three problems. {\ourshort} can sample high-quality subgraphs to constitute the mini-batch to benefit efficient training.
For the problem of \textit{neighbors explosion}, the mini-batch gradient is calculated within subgraphs so that subgraphs of acceptable size can completely avoid this problem. At the same time, the gradient does not depend on nodes outside the subgraph, which solves the \textit{node dependence} at the subgraph level. Unexpected aggregations usually occur between subgraphs. Yet, the propagation-aggregation happens within the subgraph, so the \textit{oversmoothing} can be handled.

The contributions of our work are summarized as follows:
\begin{itemize}
	\item We propose a general subgraph-based training framework {\ourt} ({\ourshort}) for GNNs. {\ourshort} not only accelerates the training speed on the large graph but also breaks through the memory bottleneck. In addition, it can effectively deal with the problem of the \textit{oversmoothing} that occurs in deep GNNs.
	\item We analyze what kind of subgraphs can support effective and efficient training. Based on the analysis, we design a novel sampling method {\oursampler} with the theoretical guarantee. 
	\item We conduct extensive experiments on different sizes of graphs to demonstrate the effectiveness of {\ourshort}. The results show the superiority of {\our} in training different GNNs (GCN \& GAT) subject to the performance and training efficiency.
\end{itemize}

\section{Related Works} \label{sec:relatedwork}

Graph neural networks ({\gnn}) aim at the machine learning tasks involving graph-structured data. The first research work extending the convolutional neural network to the graph-structured data is~\cite{sperduti1997supervised}. After that, more related works~\cite{scarselli2008graph, micheli2009neural} were introduced. More recently, \cite{bruna2013spectral} is proposed and based on spectral graph theory. Later, spatial-based ConvGNNs~\cite{gat,monti2017geometric} define graph convolutions directly based on a node's spatial relations. The spectral-based and spatial-based {\gnn} can be regarded as an information propagation-aggregation mechanism, and such mechanism is achieved by the connections and multi-layer structure. Although {\gnn} have outstanding performance, they are also limited by the problems from three aspects: \textit{node dependence}, \textit{neighbors explosion}, and \textit{oversmoothing}. Aiming at these problems, some related works have been proposed in different directions.

\textit{Node dependence}~\cite{chiang2019cluster} forces GNNs to be trained on the entire graph, which leads to the slow training process. More specifically, in each training epoch, the information aggregation involves the full graph's adjacency matrix. Thus the space complexity will be at least $O(|{\cal V}|^2)$, where $|{\cal V}|$ is the size (number of nodes) of a full graph. To deal with such a problem, \cite{chiang2019cluster, zeng2019graphsaint} apply the concept of subgraph training methods. The essence of subgraph training is to collect a batch of subgraphs from the full graph and use them during the training process. 
There are also other approaches to optimize the {\gnn} frameworks. \cite{kipf2016semi,levie2018cayleynets,liao2019lanczosnet} optimize the localized filter to reduce the time cost of training on the full graph. Further, \cite{henaff2015deep,li2018adaptive} reduce the number of learnable parameters by dimensionality reduction and residual graph Laplacian, respectively. But these approaches do not alleviate the space complexity problem.  

\textit{Neighbors expansion} makes deep GNNs difficult to being implemented. Because learning a single node requires embeddings from its neighbors, and the quantity may be explosive when a GNN goes deeper. 
Some research works deal with \textit{neighbors explosion} by neighbors sampling~\cite{hamilton2017inductive,chen2018fastgcn,chen2017stochastic}. 
In other directions, several models~\cite{gao2018large,xu2018representation} select specific neighbors based on defined metrics to avoid the explosive quantity. 
\cite{rong2019dropedge} randomly remove edges from input graphs to handle the \textit{neighbor explosion}. The works mentioned above all focus on the neighbor-level sampling but still have the same space complexity with original {\gnn}.

The problem of \textit{oversmoothing} in the GCN was introduced in~\cite{li2018deeper}. When GNNs go deep, the performance suffers from \textit{oversmoothing}, where node representations from different clusters become mixed up~\cite{zhao2019pairnorm}. The node information propagation-aggregation mechanism, can be regarded as one type of random walk within the graph. With the increasing of walking steps, the representations of nodes will finally converge to a stable status. Such convergence would impede the performance of {\gnn} and make the nodes indistinguishable in the downstream tasks. Some related works have been proposed to deal with the \textit{oversmoothing}. \cite{gresnet} comes up with the suspended animation and utilizes the residual networks to mine the advantages of deeper networks.

\section{Proposed Algorithm}\label{sec:method} 
%

\subsection{Preliminaries and Background}
For most widely used {\gnn} models (e.g., {\gcn}, {\gat}), the essence of which is aggregating feature representation for each node in the full graph, and then using the aggregated feature representation to accomplish specific tasks. Given a graph $\cal G = (\cal V, \cal E)$, the aggregation procedure of {\gnn} layers is shown as follow:
\begin{equation}\label{equ:aggre}
\footnotesize
\begin{split}
&\mb{h}^{(0)} = \mb{X}\\
&\mb{h}^{(l+1)}[i] = \sigma(\sum_{j\in {\cal N}_i} \mb{\alpha}_{ij}\cdot \mb{h}^{(l)}[j]\mb{W}^{(l)})
\end{split}
\vspace{-5pt}
\end{equation}
Here, the $\mb{X}\in \mathbb{R}^{|{\cal V}|\times F_0}$ is the input feature vectors (matrix) of all the nodes in graph $\cal G$; $\mb{h}^{(l)}[i]$ is the hidden feature of node $i$ in the $l_{th}$ layer; $\sigma$ is the non-linear function such as Relu~\cite{relu}; $\mb{W}^{(l)}$ is the learnable linear transfer matrix; $\mb{\alpha}$ is a variant of adjacency matrix, which represents different meanings according to different {\gnn} models. For example, in {\gcn} structure, $\mb{\alpha} = \widetilde{\mb{A}}$ is the normalized adjacency matrix. During the feedforward process, the hidden representations of node $i$ are updated by aggregating both its features and the local neighbors' hidden features. After layers of computing, the output representations of nodes will be delivered to the downstream tasks. The learnable weights will be optimized during the backpropagation.

The calculation in Equation~\ref{equ:aggre} uses the full graph of $\cal G$, and the full graph involved will easily lead to the concern of \textit{node dependence} and \textit{neighbors explosion}.
Since it requires the full adjacency matrix $\mb{A}$ and entire feature matrix $\mb{X}$, the size of which would be too large to deal with. 
Both the increasing size of the graph (e.g., graph with millions of nodes~\cite{chiang2019cluster}) and more sophisticated models (e.g., deeper layers~\cite{gresnet}) would aggravate the problems.


\subsection{Subgraph-based Training}\label{subsec:gnn_with_subgraph}
\begin{algorithm}[tb]
	\scriptsize
	\caption{{\ourt} for {\gnn}} 
	\begin{algorithmic}[1]
		\Require
		Graph $\cal G$; {\gnn} model $H_{\mb{W}}(\cdot)$; loss function $Loss(\cdot)$; training iteration number $T$; subgraph mini-batch size $M$
		\Ensure
		Trained $H_{\mb{W}}(\cdot)$
		\State Initialize subgraph mini-batch $batch = \{\}$
		\For {$k = 1,2,\dots , M$}
		\State ${\cal G}_k \gets {\oursampler} $\enspace/* By Algorithm~\ref{alg:ripple_walk} */
		\State $batch = batch \cup \{{\cal G}_k\}$
		\EndFor
		\For {$t = 1,2,\dots , T$}
		\State Select a subgraph from $batch$ as ${\cal G}_t$
		\vspace{0.02in}
		\State $loss = Loss( H_{\mb{W}}({\cal G}_t), \mb{y}_{{\cal G}_t})$ \enspace /* The $\mb{y}_{{\cal G}_t}$ denotes the ground truth of nodes in ${\cal G}_t$. */
		\vspace{0.02in}
		\State Update $\mb{W}$ according to the gradient $\nabla_{\mb{W}} loss$
		\EndFor 
		
		\State \Return $H_{\mb{W}}(\cdot)$
	\end{algorithmic}\label{alg:gnn_subgraph}
	
\end{algorithm}

To solve the problem of computationally expensive, an alternative approach is training the {\gnn} with {\ourshort}. The procedure of {\ourshort} is presented in Algorithm~\ref{alg:gnn_subgraph}. In Algorithm~\ref{alg:gnn_subgraph}, the subgraph mini-batch size $M$ are varying for different datasets. Generally, the value of $M$ is to satisfy that the overall number of sampled nodes (in all subgraphs) is about ten times the full graph size. Unlike the training process involving the full graph $\cal G$, {\ourshort} employs a subgraph of $\cal G$ in each training iteration. In other words, a smaller size of $\alpha$ matrix and only part of the nodes are required during each training epoch. In this way, the aggregation procedure in the $t_{th}$ training iteration is
\begin{equation}\label{equ:aggre_sub}
\vspace{-10pt}
\footnotesize
\begin{split}
&\mb{h}^{(0)} = \mb{X}_{{\cal G}_t}\\
&\mb{h}^{(l+1)}[i] = \sigma(\sum_{j\in {\cal N}^t_i} \mb{\alpha}^t_{ij}\cdot \mb{h}^{(l)}[j]\mb{W})
\end{split}
\vspace{-10pt}
\end{equation}
Here, the ${\cal G}_t = ({\cal V}_t, {\cal E}_t)$ is a subgraph of $\cal G$, where ${\cal V}_t\subseteq \cal V$ and ${\cal E}_t \subseteq \cal E$; ${\cal N}^t_i$ is the neighbor nodes set of node $i$ in ${\cal G}^t$; $\mb{\alpha}^t$ corresponds to the adjacency matrix of ${\cal G}_t$. For different training iterations, different subgraphs will be employed into Equation~\ref{equ:aggre_sub}. Comparing to the Equation~\ref{equ:aggre}, the computational complexity in Equation~\ref{equ:aggre_sub} can be reduced from $ O(|{\cal N}||{\cal V}|)$ to $O(|{\cal N}^t||{\cal V}_t|)$.

The switch from Equation~\ref{equ:aggre} to Equation~\ref{equ:aggre_sub} is similar to the change from gradient descent to mini-batch gradient descent. For {\ourshort}, the concerns are also reflected in two aspects: (1) each subgraph only contains part of the nodes; (2) subgraph is equivalent to dropping some edges, which means the dependency (connections) of nodes is incomplete. Unlike previous data type (e.g., image data), where each data sample is independent, the graph type data consists of tons of nodes connected to others.
To respond to these concerns and prove the effectiveness of {\gnn} models with subgraphs, we propose the following theorems.
\begin{thm}\label{thm:1}
	Given graph $\cal G = (\cal V,\cal E)$, assume the ${\cal V}'\subseteq \cal V$ and the nodes in ${\cal V}'$ are randomly sampled from $\cal V$; $H$ is a {\gnn} structure. The objective fucntion of training $H$ with subset nodes (${\cal V}'$) (with all neighbors) is equivalent to the objective fucntion of training with full graph, which can be presented as:
	\begin{equation}\label{equ:loss_func}
	\footnotesize \begin{split}
	\min_{H} \frac{1}{|\cal V|}\sum_{i\in \cal V} loss(H({\cal G}(i)), \mb{y}_i) \doteq \min_{H} \frac{1}{|{\cal V}'|}\sum_{j\in {\cal V}'} loss(H({\cal G}(j)), \mb{y}_j)
	\end{split}
	\end{equation}	
	where $\doteq$ denotes unbiased estimation; $loss(\cdot)$ is the selected loss function; ${\cal G}(i)$ means using node $i$'s neighbors in $\cal G$ (all neighbors) during neighbors aggregation. 
\end{thm}

\begin{proof}
	Similar to the switch from gradient descent (GD) to stochastic gradient descent (SGD), where the gradient calculated in SGD is an estimation of that in GD, the proof of Theorem~\ref{thm:1} also follows the same rule.	For the loss fucntion with full graph $\cal G$ (left part of Equation~\ref{equ:loss_func}), it has
	\begin{equation}
	\footnotesize\begin{split}
	\frac{1}{|\cal V|}\sum_{i\in \cal V} loss(H({\cal G}(i)), \mb{y}_i)
	&= \frac{1}{|\cal V|} |{\cal V}|\cdot \mathbb{E}_{i\in \cal V}[loss(H({\cal G}(i)), \mb{y}_i)]\\
	&= \mathbb{E}_{i\in \cal V} [loss(H({\cal G}(i)), \mb{y}_i)]
	\end{split}
	\end{equation}
	In the above equation, the loss function is expressed as the expectation format. Let us denote the loss function with ${\cal V}'$ (right part of Equation~\ref{equ:loss_func}) as $\mathcal{L}'$, $\mathcal{L}' =  \frac{1}{|{\cal V}'|}\sum_{j\in {\cal V}'} loss(H({\cal G}(j)), \mb{y}_j)$.	Since the nodes in ${\cal V}'$ are randomly sampled from $\cal V$, according to the statistical leanrning~\cite{annal} the $\mathcal{L}'$ is an unbiased estimation of $\mathbb{E}_{i\in \cal V} [loss(H({\cal G}(i)), \mb{y}_i)]$. Thus, we have $\mathbb{E}_{i\in \cal V} [loss(H({\cal G}(i)), \mb{y}_i)]
	\doteq \mathcal{L}'$, which also means
	\begin{equation*}
	\footnotesize\begin{split}
	\min_{H} \frac{1}{|{\cal V}|}\sum_{j\in {\cal V}} loss(H({\cal G}_(j)), \mb{y}_j) \doteq \min_{H} \frac{1}{|{\cal V}'|}\sum_{i\in {\cal V}'} loss(H({\cal G}(i)), \mb{y}_i)
	\end{split}
	\end{equation*}
\end{proof}

\vspace{-10pt}
\begin{thm}
	Under the settings in Theorem~\ref{thm:1}, the objective function of training $H$ with subgraph ${\cal G}' = ({\cal V}', {\cal E}')$ is equivalent to training with subset nodes ${\cal V}'$ (with all neighbors), which can be represented as
	\begin{equation}\label{equ:thm2_conclusion}
	\footnotesize \begin{split}
	\min_{H} \frac{1}{|{\cal V}'|}\sum_{i\in {\cal V}'} loss(H({\cal G}(i)), \mb{y}_i) \doteq \min_{H} \frac{1}{|{\cal V}'|}\sum_{j\in {\cal V}'} loss(H({\cal G}'(j)), \mb{y}_j)
	\end{split}
	\end{equation}
	where ${\cal G}'(j)$ means using node $j$'s neighbor in ${\cal G}'$ (partial neighbors).
\end{thm}
\begin{proof}
	The only difference between these two objective functions is the neighbors of each node. Since only the subgraph ${\cal G}' = ({\cal V}', {\cal E}')$ are involved during the training process, for node $i \in {\cal V}'$, only part of its neighbors are also in ${\cal V}'$. In other words, ${\cal N}'_i \subseteq {\cal N}_i$, where ${\cal N}'_i$ is the neighbor set of node $i$ in subgraph ${\cal G}'$. According to~\cite{huang2018adaptive}, the feed-forward propagation of node $i$ can be expressed as
	\begin{equation}
	\footnotesize \begin{split}
	\mb{h}^{(l+1)}[i] = \sigma(\sum_{k\in {\cal N}_i} \mb{h}^{(l)}[k]\cdot\mb{W}^{(l)}) \doteq \sigma (|{\cal N}_i|\cdot \mathbb{E}_{k\in {\cal N}_i} [\mb{h}^{(l)}[k]]\cdot \mb{W}^{(l)})
	\end{split}
	\end{equation}
	The expectation $\mathbb{E}_{k\in {\cal N}_i} [\mb{h}^{(l)}[k]]$ in the above equation can be estimated by the 
	\begin{equation}\label{equ:thm2_esti}
	\vspace{-5pt}
	\footnotesize \mathbb{E}_{k\in {\cal N}_i} [\mb{h}^{(l)}[k]] \doteq \frac{1}{|{\cal N}'_i|} \sum_{k\in {\cal N}'_i} \mb{h}^{(l)}[k]
	\vspace{-5pt}
	\end{equation}
	if the nodes in ${\cal N}'_i$ are randomly selected from ${\cal N}_i$. Given node $k\in {\cal N}_i$, we denote the possibility that node $k$ will be selected into ${\cal V}'$ as $p(k|i)$. We know that $\forall k,h \in {\cal N}_i$, $\text{p}(k|i) = \text{p}(k) = \text{p}(h) = \text{p}(h|i)$ in every step. Thus the Equation~\ref{equ:thm2_esti} can be satisfied, and
	\begin{equation}
	\footnotesize H({\cal G}(i)) \doteq H({\cal G}'(i)), \; \forall i\in {\cal V}'
	\end{equation}
	Therefore, the Equation~\ref{equ:thm2_conclusion} can hold.
\end{proof}

From the analysis above, to achieve the equivalent training effect, the subgraphs should possess:
\begin{itemize}
	\item \textbf{randomness}: The randomness contains two aspects: (1) each node has the same probability to be selected; (2) for any node, its neighbors own the same probability to be selected. Randomness can help eliminate the \textit{neighbors explosion} problem.
	\item \textbf{connectivity}: The subgraph should preserve the connectivity in the full graph. The connectivity of each subgraph should be high enough to preserve the connectivity in the full graph. This corresponds to the \textit{node dependence} problem. 
\end{itemize}

In this way, even though each subgraph cannot singly cover all the nodes and structure information in $\cal G$, the batch of subgraphs can help achieve the same object as the full graph as long as each subgraph satisfies the randomness and connectivity characteristics. To follow these two characteristics, we propose the {\oursampler} algorithm.

\subsection{{\our} Subgraph Sampling}\label{subsec:ripple_walk}
\begin{algorithm}[tb]
	\scriptsize
	\caption{{\oursampler}} 
	\begin{algorithmic}[1] 
		\Require
		Target graph $\cal G = (\cal V, \cal E)$; expansion ratio $r$; target subgraph size $S$
		\Ensure
		Subgraph ${\cal G}_k$
		\State Initiate ${\cal G}_k = ({\cal V}_k, {\cal E}_k)$ with ${\cal V}_k = \phi$
		\State Randomly select the initial node $v_{s}$, add $v_{s}$ into the ${\cal G}_k$
		\While {$|{\cal V}_k| < S$}
		\State  $NS = \{n| (n, j)\!\in\!{\cal E}, j\!\in\!{\cal V}_k, n\!\in\!{\cal V}\setminus{\cal V}_k\}$\enspace /* Get neighbor nodes set of current ${\cal V}_k$ */
		\State Randomly select $r$ of nodes in $NS$, add them into the ${\cal V}_k$ 
		\EndWhile
		\State \Return ${\cal G}_k$
	\end{algorithmic}\label{alg:ripple_walk}
\end{algorithm}

\begin{figure*}[t]
	\vspace{-30pt}
	\centering
	\includegraphics[width=0.9\textwidth, height=0.17\textheight]{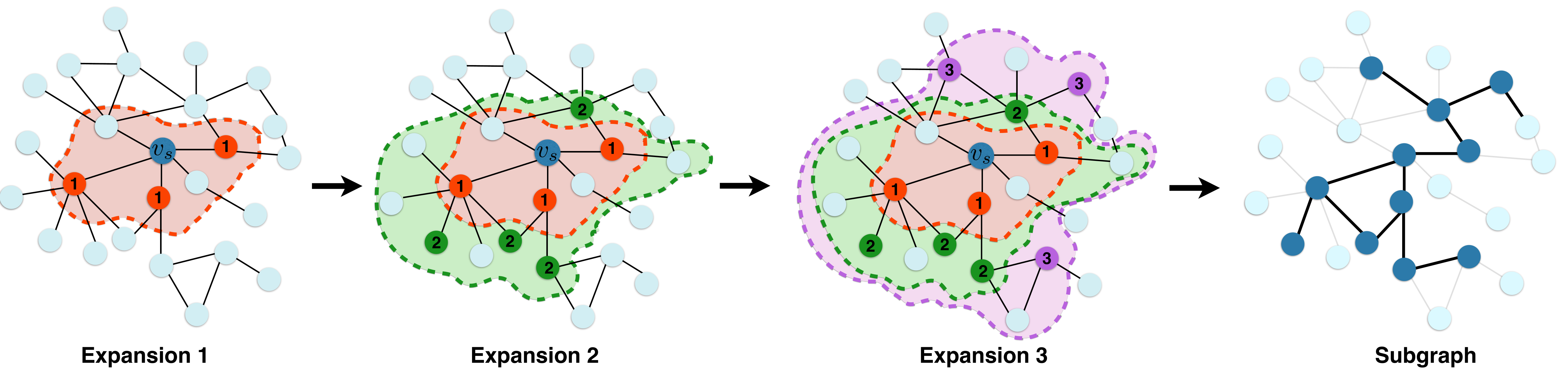}
	\vspace{-10pt}
	\caption{Ripple Walk Sampling. In each expansion, the expansion ratio $r = 0.5$, the background color region represents the neighbor set, the colored nodes represents the truly sampled nodes. (Best viewed in color)}
	\label{fig:ripple_walk}
	\vspace{-15pt}
\end{figure*}

The {\oursampler} algorithm is shown in Figure~\ref{fig:ripple_walk}. For the subgraph ${\cal G}_k$, it is initialized with a random node $v_{s}$, then expands along with the connections among nodes. After multiple expansion steps (sampling), the subgraph with a specific size (e.g., $S$) will be returned. During each expansion step, the neighbor set (shown by the background color region in Figure~\ref{fig:ripple_walk}) contains the potential nodes to be sampled. Then $r$ (e.g., $r = 0.5$) of the nodes (shown by colored nodes in Figure~\ref{fig:ripple_walk}) in neighbor set will be added into the current subgraph. Here, $r$ is the expansion ratio, which means the ratio of nodes in the neighbor set to be sampled in the current step. Such an expansion process operates like the ``ripple'' on the water. More details of {\oursampler} are exhibited in Algorithm~\ref{alg:ripple_walk}.

From the analysis in Section~\ref{subsec:gnn_with_subgraph}, we conclude that it is ideal if the sampled subgraphs possess both randomness and connectivity. The {\oursampler} strategy can maintain randomness by randomly sampling initial node and random expansion, while the expansion along edges can guarantee connectivity. In the following part, we will show the advantages of {\oursampler} algorithm concerning those two characteristics.

During the expansion of {\oursampler}, $r$ determines the range of the subgraph. When $r \rightarrow 0$, it can be regarded as random sampling. For the randomly sampled subgraph, the connectivity might be too low to reproduce the global structure in the full graph. To show the advantages of {\oursampler} compared with a random sample, we first state the following theorem:
\begin{thm}\label{thm:random}
	From graph $\cal G$, ${\cal G}_k = ({\cal V}_k, {\cal E}_k)$ is the subgraph generated by {\oursampler}, while ${\cal G}_r = ({\cal V}_r, {\cal E}_r)$ is the randomly sampled subgraph. Then $\forall i, j\in {\cal V}_r$ and $\forall m, l\in {\cal V}_k$, it has
	\begin{equation}
	\footnotesize \text{p}((i, j)\in {\cal E}_r)\leq\text{p}((m,l)\in {\cal E}_k)
	\end{equation}
	where $\text{p}(\cdot)$ denotes the probability.
\end{thm}

\begin{proof}
	According to Algorithm~\ref{alg:ripple_walk}, in each sampling step, for $\forall i\in NS$, there $\exists j\in{\cal V}_k$ having $(i, j)\in {\cal E}$. Thus when {\our} adds one node into the subgraph, one edge will be added into ${\cal E}_k$ as well. For ${\cal G}_r$, when a new node is selected into the subgraph, possibly there is no new edge added. For subgraphs with the same number of nodes, more connections will selected into ${\cal E}_k$ comparing to ${\cal E}_r$. Therefore, we have $\text{p}((i, j)\in {\cal E}_r)\leq\text{p}((m,l)\in {\cal E}_k)$.
\end{proof}
\noindent From Theorem~\ref{thm:random}, it is obvious that {\our} can join more connections during the sampling process. Thus the connectivity of subgraphs by {\oursampler} is higher than the randomly sampled subgraphs. 

Similar to the {\our}, Breadth-First-Search (BFS) is a graph search algorithm that expands from one central node and traverses the whole neighbor set. Essentially, BFS is equivalent to {\oursampler} with $r\rightarrow 1$. Different from Ripple Walk Sampler, BFS cannot guarantee the randomness of node sampling: for BFS, once the initial node $v_s$ and the target subgraph size $S$ are certain, the nodes to be selected into the subgraph have been determined. 
In fact, if BFS satisfies the randomness mentioned in Section~\ref{subsec:gnn_with_subgraph}, the subgraph cannot be determined by the initial node. On the other hand, {\oursampler} can maintain randomness. Except for the random initial node, the neighbor nodes in each step are sampled randomly. Even starting from the same initial node, {\oursampler} can still generate different subgraphs. 

From the analysis above, {\oursampler} not only can keep the randomness of the sampled subgraphs, but also maintain a relatively high level of connectivity. With these two characteristics, {\ourshort} can solve the \textit{neighbors explosion} and \textit{node dependence} problems, meanwhile reproduce the information in the full graph. The selection of expansion ratio is important, and we provide the analyses of parameter $r$ in Section~\ref{sec:parameter_r}. 

\vspace{-5pt}
\subsection{For Deeper Graph Networks}
The commonly used {\gnn} only involve no more than two layers. According to~\cite{li2018deeper}, each {\gcn} layer can be regarded as one type of Laplacian smoothing, which essentially computes the features of nodes as the weighted average of itself and its neighbors'. In other words, GNN structures with much deeper layers will repeatedly carry out Laplacian smoothing, and features of nodes will finally converge to the global steady states. Such smoothing will undermine the learning ability of {\gnn}. This point of view also corresponds to the concepts of over smoothing and mixing time in~\cite{rong2019dropedge,lovasz1993random}.

By applying {\gnn} with subgraphs, we will prove that {\ourshort} can eliminate the problem of converging to global steady states. Subsequently, {\gnn} with deeper layers can achieve better learning capability. We will give the following definition and assumption. 
\begin{defn}
	(Node distribution): In graph $\cal G = (\cal V, \cal E)$, $\mb{h}^{(0)}[i]\sim \mathcal{D}_i$ denotes that the feature representation of node  $i \in \cal V$ is under the distribution $\mathcal{D}_i$.
\end{defn}

In graph $\cal G$, each node is under a corresponding distribution. While different nodes might own different labels, we assume that nodes within the same class share similar distributions. The assumption can be presented as
\begin{prop}\label{lemma:kl1}
	In graph $\cal G = (\cal V, \cal E)$ with $i, j, k \in \cal V$, if $\mb{y}_i = \mb{y}_j$ and $\mb{y}_i \not= \mb{y}_k$, then we assume
	\begin{equation}\label{equ:kl1}
	\footnotesize \mathbb{E}_{\mb{y}_i = \mb{y}_j}[KL({\cal D}_i, {\cal D}_j)] \leq \mathbb{E}_{\mb{y}_i \not= \mb{y}_k}[KL({\cal D}_i, {\cal D}_k)]
	\end{equation}
	where $	\mathbb{E}$ is the expectation and $KL$ is the Kullback-Leibler divergence (KL divergence).
\end{prop}
KL divergence is a measure of the difference between two probability distributions. To be simplified, here we call it the KL divergence of two nodes. 
It is easy to understand since the same labeled nodes are more likely to share information (features) that comes from similar distributions.

According to Equation~\ref{equ:aggre}, the computation in each {\gnn} layer is the weighted averaging of each node's neighbors. If we ignore the linear transform by $\mb{W}^{(0)}$, from the node distribution view, it can be written as
\begin{equation}
\footnotesize \mb{h}^{(1)}[i] = \sum_{k\in {\cal N}_i} \alpha_{ik} \mb{h}^{(0)}[k] \sim \text{Joint}({\cal D}_{k\in {\cal N}_i}) \triangleq {\cal D}^{(1)}_i
\end{equation}
where $\text{Joint}({\cal D}_{k\in {\cal N}_i})$ means the weighted average distribution of each ${\cal D}_k$, and we denote $\text{Joint}({\cal D}_{k\in {\cal N}_i})$ as ${\cal D}^{(1)}_i$. Through one layer of calculation, the new hidden representation of node $i$ will be under the $\text{Joint}({\cal D}_{k\in {\cal N}_i})$ distribution. After $l$ layers, we denote it as $\mb{h}^{(l)}[i] \sim {\cal D}^{(l)}_i$. 

\begin{thm}\label{thm:deeper}
	For full graph $\cal G = (\cal V, \cal E)$ and subgraph mini-batch  $\{{\cal G}_1, {\cal G}_2, \dots, {\cal G}_M\}$ generated by {\oursampler}. Assume the nodes within the local parts are more likely to share the same label. Let $i, j\in {\cal V}$ and $m, n\in {\cal V}_k$, ${\cal G}_k = ({\cal V}_k, {\cal E}_k)\in \{{\cal G}_1, {\cal G}_2, \dots, {\cal G}_M\}$. Then,
	\begin{equation}
	\footnotesize \mathbb{E}_{m, n\in {\cal V}_k}[KL({\cal D}_m, {\cal D}_n)] \leq \mathbb{E}_{i ,j\in \cal V}[KL({\cal D}_i, {\cal D}_j)]
	\end{equation}
\end{thm}

\begin{proof}
	According to {\oursampler}, ${\cal G}_k$ only covers part of local nodes in $\cal G$. Thus for $\forall i, j \in {\cal V}$ and $\forall m,n \in {\cal V}_k$, $\text{p}(\mb{y}_i = \mb{y}_j) \leq \text{p}(\mb{y}_m = \mb{y}_n)$. Therefore,
	\begin{equation}
	\footnotesize \begin{split}
	&\mathbb{E}_{m, n\in {\cal V}_k}[KL({\cal D}_m, {\cal D}_n)]\\
	=& \text{p}(\mb{y}_m = \mb{y}_n)\cdot \mathbb{E}_{\mb{y}_m = \mb{y}_n}KL({\cal D}_m, {\cal D}_n)\\
	&+ \text{p}(\mb{y}_m \not= \mb{y}_n)\cdot \mathbb{E}_{\mb{y}_m \not= \mb{y}_n}KL({\cal D}_m, {\cal D}_n)\\
	\leq& \text{p}(\mb{y}_i = \mb{y}_j)\cdot \mathbb{E}_{\mb{y}_i = \mb{y}_j}KL({\cal D}_i, {\cal D}_j)\\
	&+ \text{p}(\mb{y}_i \not= \mb{y}_j)\cdot \mathbb{E}_{\mb{y}_i \not= \mb{y}_j}KL({\cal D}_i, {\cal D}_j)\\
	=& \mathbb{E}_{i ,j\in \cal V}[KL({\cal D}_i, {\cal D}_j)]
	\end{split}
	\end{equation}
	\vspace{-5pt}
\end{proof}

From Theorem~\ref{thm:deeper}, the distribution similarity of nodes in the subgraph is higher than that in the full graph. It is easy to understand since each subgraph generated by {\our} contains nodes from the local part. The randomness ensures that different subgraphs could cover other local parts of the full graph for the subgraph mini-batch. Subsequently, with the increasing of $l$, the distribution ${\cal D}^{(l)}$ in each subgraph will converge to different steady states: since each subgraph possesses different nodes and structures. Compared with the global steady state, different steady states correspond to the local information within various subgraphs, which can help improve the learning capacity of deep {\gnn}.

\section{Experiments}\label{sec:experiment}

To show the effectiveness and efficiency of {\ourshort}, extensive experiments have been conducted on real-world datasets. 
We aim to answer the following evaluation questions based on experimental results together with the detailed analysis:
\begin{itemize}
	\item \textbf{Question 1}: Can {\ourshort} break through the memory bottleneck in order to handle \textit{node dependence}?
	\item \textbf{Question 2}: Can {\ourshort} accelerate the training speed on large graphs when facing \textit{neighbors explosion}?
	\item \textbf{Question 3}: Can {\oursampler} provide powerful subgraphs to support effective and efficient training?
	\item \textbf{Question 4}: Can {\ourshort} tackle the \textit{oversmoothing} problem occuring in deep GNNs?
	
\end{itemize}

\subsection{Experiment Settings}
\subsubsection{Datasets}

We test our algorithms on 5 datasets. Three of them are standard citation network benchmark datasets: Cora, Citeseer, and Pubmed~\cite{sen2008collective}. Flickr~\cite{zeng2019graphsaint,mcauley2012image} is built by forming links between images sharing common metadata from Flickr. Edges are formed between images from the same location, submitted to the same gallery, group, or set, images sharing common tags, images taken by friends, etc. For labels, Zeng et al.~\cite{zeng2019graphsaint} scan over the 81 tags of each image and manually merged them into 7 classes. Each image belongs to one of the 7 classes. Reddit~\cite{hamilton2017inductive} is a graph dataset constructed from Reddit posts. In this case, the node label is the community or “subreddit” that a post belongs to.

These datasets involve both the transductive task and inductive task. The transductive task in our experiments is semi-supervised node classification on one graph; the inductive task is the node classification on multiple graphs. The information of them is presented in Table~\ref{tab:dataset}. The label rate in the table means the ratio of training data.
\begin{table}[t]
	\vspace{-20pt}
	\scriptsize 
	\caption{Datasets in Experiments}
	\vspace{-8pt}
	\renewcommand\arraystretch{1}
	\centering
	\begin{threeparttable}
		\begin{tabular}{l c c c c c}
			\toprule[1.5pt]
			\multirow{2}*{}&\multicolumn{3}{c}{Transductive}&\multicolumn{2}{c}{Inductive}\\
			
			\cline{2-6}
			&Cora&Citeseer&Pubmed&Flickr&Reddit\\
			\hline
			\# Nodes&2708&3327&19717&89250&232965\\
			\# Edges&5429&4732&44338&899756&11606919\\
			\# Features&1433&3703&500&500&602\\
			\# classes&7&6&3&7&41\\
			Label rate&0.052&0.036&0.003&0.6&0.6\\
			\bottomrule[1.5pt]
		\end{tabular}
		
	\end{threeparttable}
	\vspace{-15pt}
	\label{tab:dataset}
\end{table}
\subsubsection{GNNs Models for Training}

We have applied the {\ourshort} to train {\gcn} and {\gat} models, respectively, which both are representative and widely used GNN models. The default models contain two layers. The hidden layers involve different sizes based on different datasets: for Cora, Citeseer, and Pubmed, the hidden layers' size is 32; for Flickr and Reddit, the hidden size is 128 for {\gcn} layer and 8 for {\gat} layer. The dropout rate is 0.5 for Cora, Citeseer, Pubmed, and 0.1 for Flickr, Reddit. We employ the {\adam}~\cite{adam} as the optimizer. The learning rate is 0.01, with weight decay as $5\times 10^{-4}$.

\subsubsection{Comparison Methods}
We compare the GNNs trained by {\ourshort} with state-of-the-art baseline methods:

\noindent
\textbf{\textit{Comparison Models}}
\begin{itemize}
	\item \textbf{{\gcn}}~\cite{kipf2016semi}: GCN is a semi-supervised method for the node classification, which operates on the whole graph.
	\item \textbf{{\gat}}~\cite{gat}: GAT is an attention-based graph neural network for the node classification. 
	\item \textbf{GraphSAGE}~\cite{hamilton2017inductive}: GraphSAGE is a general inductive framework that leverages node feature information to generate node embeddings for unseen data efficiently.
	\item \textbf{Cluster-GCN}~\cite{chiang2019cluster}: Cluster-GCN is a suitable framework for SGD-based training. It samples a block of nodes that associate with a dense subgraph identified by a graph clustering algorithm and restricts the neighborhood search within this subgraph.
	
\end{itemize}

\noindent
\textbf{\textit{Comparison Samplers}}
\begin{itemize}
	\item \textbf{{\oursampler}}: {\oursampler} is the sampler proposed in this paper.
	\item \textbf{Random}: The Random sampler randomly selects a certain number from all nodes to form a subgraph for training.
	\item \textbf{BFS}: The BFS sampler performs a breadth-first search from the starting node to select subgraphs.	
\end{itemize}

For the Cluster-GCN, the number of clusters is set to make the average subgraph size the same as {\oursampler}. Meanwhile, the self-comparison is conducted among {\oursampler}, BFS sampling, and Random sampling strategies. For these samplers, the sampled subgraph size on Cora and Citeseer datasets is $S = 1500$, on Pubmed and Flickr $S = 3000$, on Reddit $S = 5000$. For the subgraph mini-batch size $M$, it can be varying for different datasets. Generally, we make the overall sampled nodes at least twice of the nodes in a full graph, which also means $M\times S \simeq 10 |{\cal V}|$.  In the following parts, the ``GCN + Random / BFS / RWT" denotes {\gcn} model training with subgraphs from Random sampler, BFS sampler and {\oursampler}, respectively.


We run the experiments on the Server with 3 GTX-1080 ti GPUs, and all codes are implemented in Python. 
Code is available at:
\href{https://github.com/YuxiangRen/RippleWalk}{https://github.com/YuxiangRen/RippleWalk}.

\subsection{Experimental Results with Analysis}
\begin{table}[t]
	\vspace{-15pt}
	\scriptsize
	\renewcommand\arraystretch{1}
	\centering
	\begin{threeparttable}
		\caption{Test accuracy results on all datasets}
		\vspace{-8pt}
		\begin{tabular}{l c c c c c}
			\toprule[1.5pt]
			\multirow{2}*{\textbf{Methods}}&\multicolumn{3}{c}{Transductive}&\multicolumn{2}{c}{Inductive}\\
			
			\cline{2-6}
			&Cora&Citeseer&Pubmed&Flickr&Reddit\\
			
			\hline
			GraphSAGE&0.7660&0.6750&0.7610&0.4030&0.9300\\ 
			Cluster-GCN&0.682&0.628&0.7947&0.4097&\textbf{0.9523}\\
			\hline
			{\gcn}&0.815&0.7030&0.7890&0.4400&0.9333\\
			\hdashline
			{\gcn} + Random&0.7945&0.687&0.7345&0.4713&0.8243\\
			{\gcn} + BFS&0.8144&0.7079&0.7971&0.4754&0.8123\\
			{\gcn} + RWT&\textbf{0.825}&\textbf{0.7127}&\textbf{0.8259}&\textbf{0.4797}&0.9495\\
			\hline
			{\gat}&\textbf{0.8300}&0.7130&0.7903&-&-\\
			\hdashline
			{\gat} + Random&0.7921&0.6607&0.6765&0.4534&0.6452\\
			{\gat} + BFS&0.7756&0.6500&0.7080&0.4642&0.7297\\
			{\gat} + RWT&0.7994&\textbf{0.7212}&\textbf{0.8210}&\textbf{0.4724}&\textbf{0.8699}\\		
			
			\bottomrule[1.5pt]
		\end{tabular} 
		
		\begin{tablenotes}
			\footnotesize
			\item``-'' insufficient memory.
		\end{tablenotes}
	\end{threeparttable}
	\label{tab:results_all}
	\vspace{-15pt}
\end{table}
\subsubsection{Task Performance Analysis}
The overall task performance of {\ourshort} and comparison methods are exhibited in Table~\ref{tab:results_all}. The most important thing about the training framework {\ourshort} is to ensure the training effect of GNN. On the premise of competitiveness or better task performance, the advantages of time and space can be meaningful. We can first observe that {\gcn} and {\gat} with {\ourshort} outperforms plain {\gcn} and {\gat} in most of the cases. Here, both {\gcn} and {\gat} contain two layers. For the {\gcn} model, {\ourshort} has better overall performance than GraphSAGE and Cluster-GCN; for the {\gat}, even in some cases when training with a full graph cannot be executed due to limited memory space (e.g., on Flickr and Reddit), {\gat} with {\ourshort} can successfully run and achieve high performance. For the self-comparison, {\oursampler} achieves the best results compared with random and BFS sampling. Generally speaking, {\gnn} with {\ourshort} can achieve the same level or even better testing performance compared with other popular baseline methods. Meanwhile, the advantages of {\ourshort} on space complexity are significant. In the following part, we will answer four evaluation questions mentioned before to validate the efficiency of {\ourshort}.
\subsubsection{Space-consuming Analysis}
\begin{table}[t]
	\vspace{-20pt}
	\caption{Memory Space Usage (The unit is MB)}
	\vspace{-8pt}
	\scriptsize
	\renewcommand\arraystretch{1}
	\centering
	\begin{tabular}{l c c c c c}
		\toprule[1.5pt]
		&Cora&Citeseer&Pubmed&Flickr&Reddit\\
		\hline
		{\gcn}&535&605&2057&30392&212003\\
		{\gcn} + RWT&\textbf{509}&\textbf{587}&\textbf{1235}&\textbf{922}&\textbf{1101}\\
		\hdashline
		{\gat}&6921&10277&11868&243089&243089\\
		{\gat} + RWT&\textbf{2121}&\textbf{2469}&\textbf{2629}&\textbf{12000}&\textbf{12080}\\
		\bottomrule[1.5pt]
	\end{tabular}
	\label{tab:space}
	\vspace{-15pt}
\end{table}
One of the critical advantages of {\ourshort} compared with plain {\gnn} training is less space-consuming. To answer \textbf{Question 1}, we first compare the memory space usage and show them in Table~\ref{tab:space}. It is obvious that training {\gnn} with {\ourshort} requires less memory space than plain {\gnn}. Especially for {\gat}, the space usage of {\ourshort} is much less than using the full graph. Therefore, when training plain {\gnn} is too space-consuming to run, {\ourshort} can help conduct the training process of {\gnn} and the performance can be guaranteed. Meanwhile, the less space-consuming of {\ourshort} enables the {\gnn} to be carried out on GPUs, which can further accelerate the training. In this part, we do not compare {\ourshort} with other baseline methods, since for the neighbor node sampling methods such as GraphSAGE~\cite{hamilton2017inductive}, FastGCN~\cite{chen2018fastgcn} or methods in~\cite{chen2017stochastic, huang2018adaptive} do not apply the subgraph training concept. Thus their space complexity will be the same with the plain {\gnn}. Therefore we can conclude that while maintaining the high-level testing performance, {\ourshort} can significantly minimize the space-consuming when running {\gnn}. In this way, {\ourshort} break the memory bottleneck when facing huge graphs. The experimental results also validate our previous theoretic analyses about both effectiveness and efficiency.  
\subsubsection{Time-consuming Analysis}
\begin{table}[h]
	\vspace{-20pt}
	\caption{Training Time (The unit is second)}
	\vspace{-8pt}
	\scriptsize
	\renewcommand\arraystretch{1}
	\centering
	\begin{tabular}{l c c c c c}
		\toprule[1.5pt]
		&Cora&Citeseer&Pubmed&Flickr&Reddit\\
		\hline
		{\gcn}&4.573&1.968&61.90&1161.92&25370\\
		{\gcn} + RWT&\textbf{1.964}&\textbf{1.826}&\textbf{8.698}&\textbf{1.179}&\textbf{7.722}\\
		\hdashline
		{\gat}&413.3&500.1&-&-&-\\
		{\gat} + RWT&\textbf{71.44}&\textbf{47.06}&\textbf{139.4}&\textbf{68.06}&\textbf{2614}\\
		\bottomrule[1.5pt]
	\end{tabular}
	\begin{tablenotes}
		\footnotesize
		\item``-'' insufficient memory.
	\end{tablenotes}
	\vspace{-12pt}
	\label{tab:time}
\end{table}
To answer \textbf{Question 2}, we further record the duration time of the training process and present it in Table~\ref{tab:time}. The running time of {\gcn} on Flickr and Reddit datasets, {\gat} + {\ourshort} on Reddit dataset is based on the CPU server. All other convergence time is recorded on GPUs. By applying {\ourshort}, the running time of {\gnn} models can be reduced by a large margin. Thus the {\ourshort} on {\gnn} not only requires less memory space but also can accelerate the convergence of the training process. Such results can be explained intuitively since each training epoch of {\ourshort} involves less space and computation complexity. While sharing the same computational quantity, {\ourshort} possesses more training iterations than plain {\gnn} training.
\subsubsection{Sampler performance Analysis}
To answer \textbf{Question 3}, we conduct comparison experiments among different samplers. From Table~\ref{tab:results_all}, we can notice that Ripple Walk Sampler outperforms other samplers on all five datasets. These results verify that the subgraphs provided by Ripple Walk Sampler are more beneficial for graph model training. Random sampler leans towards randomness, while BFS sampler focuses on connectivity too much. They can not consider randomness and connectivity simultaneously, yet Ripple Walk Sampler can do so. It also validates our theoretical analysis in Section~\ref{subsec:gnn_with_subgraph}.
\subsubsection{For Deeper Graph Networks}

\begin{figure}[t]
	\vspace{-20pt}
	\centering
	\includegraphics[width=0.3\textwidth, height= 0.17\textheight]{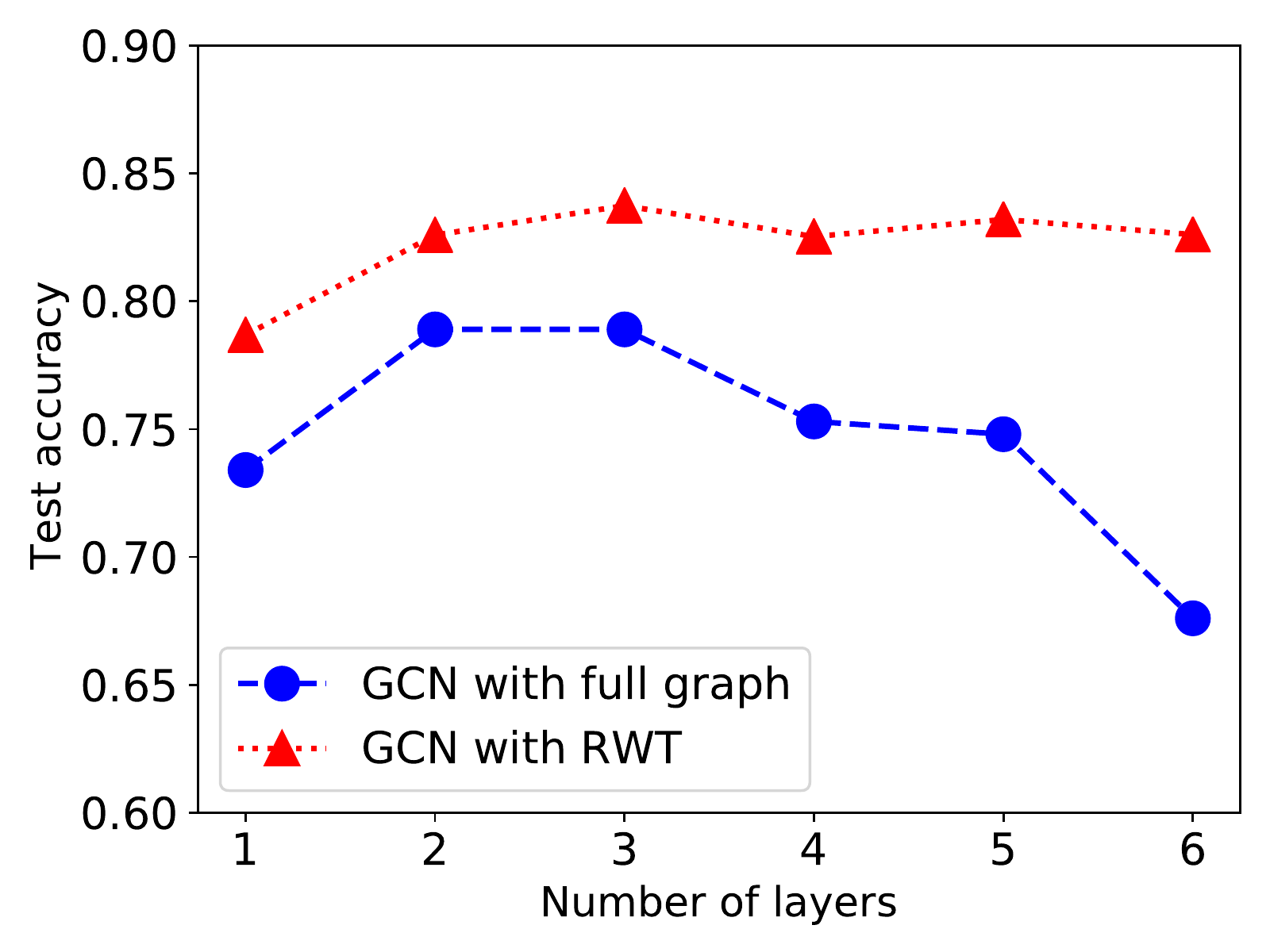}
	\vspace{-10pt}
	\caption{{\gcn} with Deeper Layers}\label{fig:deeper}
\end{figure}

We test {\gcn} models with different numbers of layers to answer \textbf{Question 4}. The results are shown in Figure~\ref{fig:deeper}. We show the results on the Pubmed dataset, and the experimental results are consistent in other datasets. We can observe that {\gcn} with {\ourshort} achieves better performance than plain {\gcn} on the test loss and accuracy. Besides, with the structure goes deeper, even when the performance of {\gcn} decreases, {\gcn} with {\ourshort} achieves higher performance. Thus with the support of {\ourshort}, the problem of \textit{oversmoothing} can be eliminated, and {\gnn} models can be designed with deeper structure. From such phenomenon, deeper {\gnn} models such as~\cite{gresnet} can benefit from the {\ourshort}.

\begin{figure}[t]
	\vspace{-15pt}
	\centering
	\includegraphics[width=0.3\textwidth, height= 0.15\textheight]{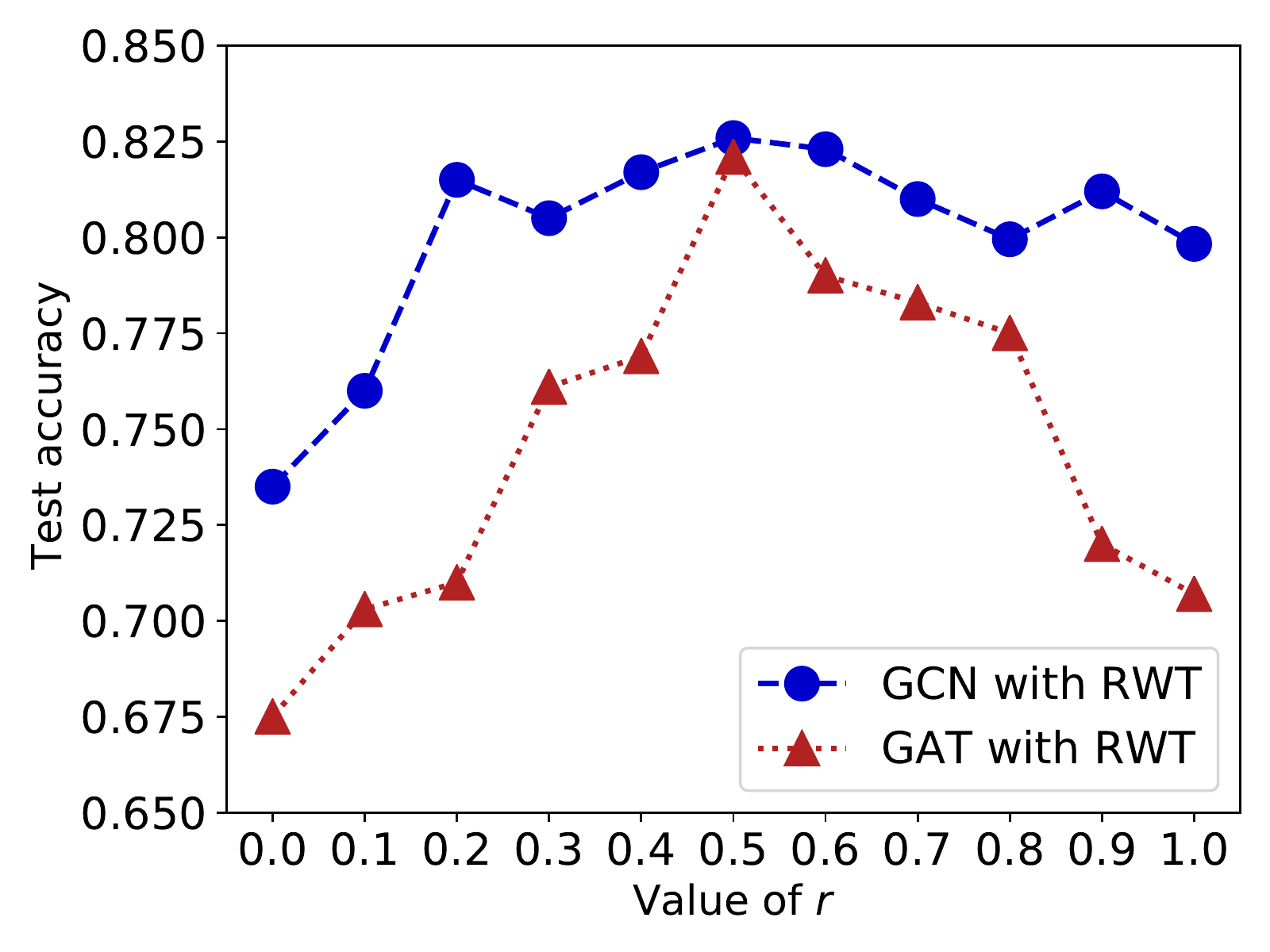}
	\vspace{-10pt}
	\caption{{\gnn} with Different $r$}\label{fig:ratio}
	\vspace{-15pt}
\end{figure}
\vspace{-5pt}
\subsection{Parameter Analysis}

\subsubsection{Expansion Ratio $r$ Analysis}\label{sec:parameter_r}
To verify the analysis of $r$ in Subsection~\ref{subsec:ripple_walk}, we implement experiments of {\oursampler} with different expansion ratios, and present the results in Figure~\ref{fig:ratio}. We show the results on Pubmed, and the experimental results are consistent in all datasets. According to previous analysis in  Subsection~\ref{subsec:ripple_walk}, $r\rightarrow 0$ or $r\rightarrow 1$ do not help maintain the randomness and connectivity characteristics in subgraphs. From the results we can observe that when $r = 0.5$, {\ourshort} achieve the best performance and the performance decreases when $r\rightarrow 1$ or $r\rightarrow 0$. Therefore, the results verify our previous analysis, The subgraphs sampled by {\oursampler} consider both randomness and connectivity, which are beneficial to subgraph-based training for GNNs.

\subsubsection{Subgraph Mini-batch Size $M$ Analysis}

As mentioned above, we have applied the subgraph mini-batch size to satisfy the $M\times S \simeq 10 |{\cal V}|$. Give a fixed subgraph size $S$, we also  attempt different values of $M$ and present the results on Pubmed dataset with $S = 3000$ and $M\in \{10, 20, \dots, 100\}$ in Figure~\ref{fig:mini-batch_size}.   Both GCN and GAT with RWT have relatively worse performance when $M$ is small (e.g., $M\in \{10,20,30\}$). With the increase of $M$, the performance becomes better and finally maintains at a high level.  It is intuitive that with a smaller subgraph mini-batch size $M$, less information of the full graph can be covered. While $M$ rising, more subgraphs combinations will be selected to draw the information (e.g., connections) within the full graph. On the other hand, along with increasing $M$, subgraphs with similar nodes and edges could be sampled. Such similar subgraphs may cause redundant information extraction from the full graph, and the performance will finally converge instead of improving continuously.
\begin{figure}[t]
	\vspace{-5pt}
	\centering
	\includegraphics[width=0.3\textwidth, height= 0.17\textheight]{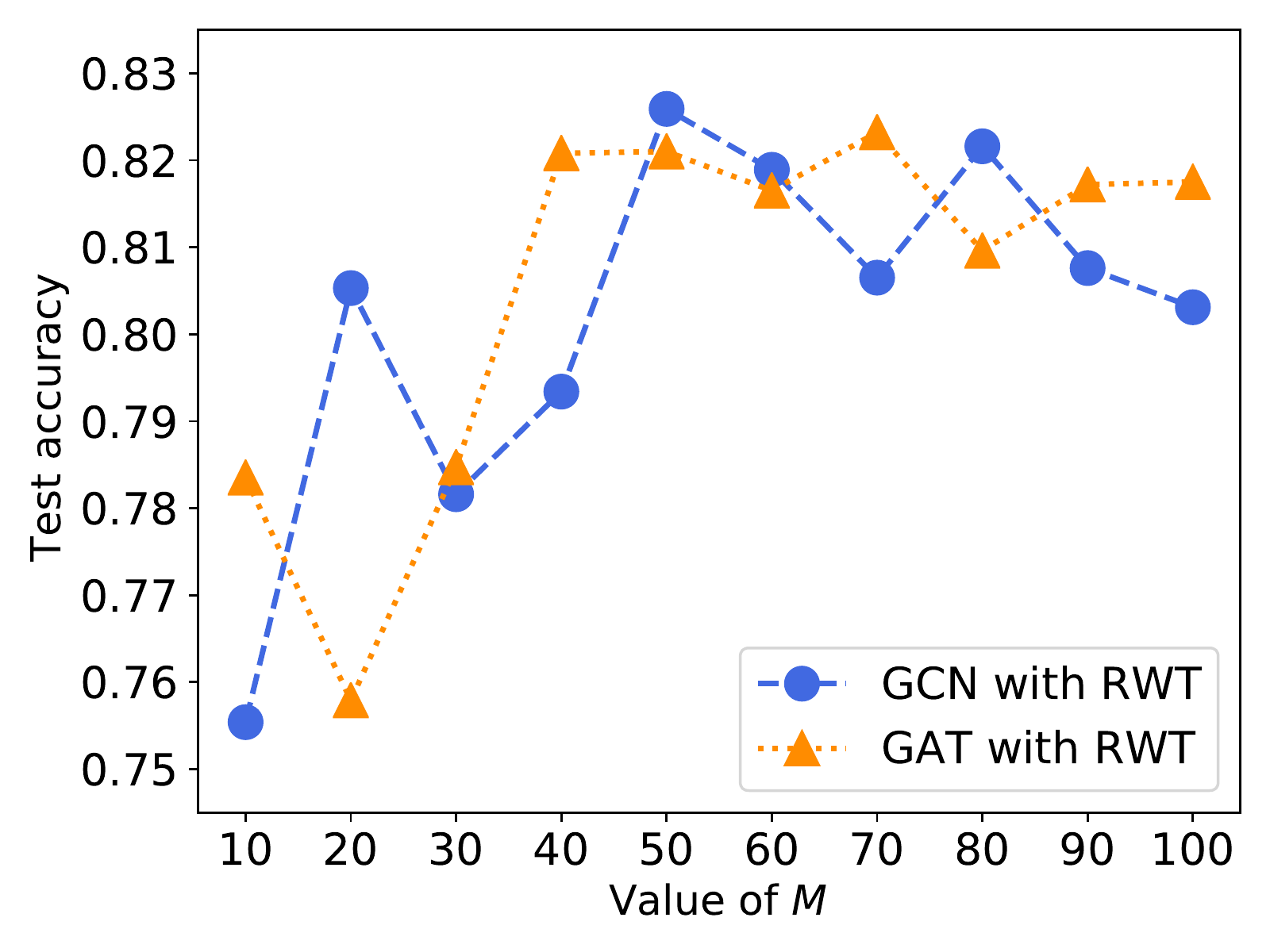}
	\vspace{-10pt}
	\caption{{\gnn} with Different values of $M$}\label{fig:mini-batch_size}
	\vspace{-20pt}
\end{figure}

\vspace{-5pt}
\section{Conclusion}\label{sec:conclusion}
In this paper, we have introduced a subgraph-based training framework {\ourshort} for {\gnn}, which combines the idea of training {\gnn} with mini-batch subgraphs and a novel subgraph sampling method Ripple Walk Sampler. We analyze the effectiveness and efficiency of the {\ourshort} and prove it from the theoretical perspective. Extensive experiments demonstrate that {\ourshort} achieves the same level or even better performance, but less training time and device memory space are required. At the same time, {\ourshort} can help relieve the problem of \textit{oversmoothing} when models go deeper, enabling the GNNs to have stronger learning power and potential.
\vspace{-5pt}

\section{Acknowledgement}\label{sec:ack}
\vspace{-5pt}
This work is partially supported by NSF through grant IIS-1763365.


\bibliographystyle{plain}
\bibliography{reference}

\begin{thebibliography}{10}

\bibitem{annal}
George~W. Brown.
\newblock On small-sample estimation.
\newblock {\em The Annals of Mathematical Statistics}, 18(4):582--585, 1947.

\bibitem{bruna2013spectral}
Joan Bruna, Wojciech Zaremba, Arthur Szlam, and Yann LeCun.
\newblock Spectral networks and locally connected networks on graphs.
\newblock {\em arXiv:1312.6203}, 2013.

\bibitem{chen2017stochastic}
Jianfei Chen, Jun Zhu, and Le~Song.
\newblock Stochastic training of graph convolutional networks with variance
  reduction.
\newblock In {\em Proceedings of the 35th International Conference on Machine
  Learning}, 2018.

\bibitem{chen2018fastgcn}
Jie Chen, Tengfei Ma, and Cao Xiao.
\newblock Fastgcn: fast learning with graph convolutional networks via
  importance sampling.
\newblock {\em arXiv:1801.10247}, 2018.

\bibitem{chiang2019cluster}
Wei-Lin Chiang, Xuanqing Liu, Si~Si, Yang Li, Samy Bengio, and Cho-Jui Hsieh.
\newblock Cluster-gcn: An efficient algorithm for training deep and large graph
  convolutional networks.
\newblock In {\em Proceedings of the 25th ACM SIGKDD International Conference
  on Knowledge Discovery \& Data Mining}, pages 257--266, 2019.

\bibitem{FBSB17}
Alex Fout, Jonathon Byrd, Basir Shariat, and Asa Ben-Hur.
\newblock Protein interface prediction using graph convolutional networks.
\newblock In {\em Advances in neural information processing systems}, pages
  6530--6539, 2017.

\bibitem{gao2018large}
Hongyang Gao, Zhengyang Wang, and Shuiwang Ji.
\newblock Large-scale learnable graph convolutional networks.
\newblock In {\em Proceedings of the 24th ACM SIGKDD International Conference
  on Knowledge Discovery \& Data Mining}, pages 1416--1424, 2018.

\bibitem{hamilton2017inductive}
Will Hamilton, Zhitao Ying, and Jure Leskovec.
\newblock Inductive representation learning on large graphs.
\newblock In {\em Advances in neural information processing systems}, pages
  1024--1034, 2017.

\bibitem{henaff2015deep}
Mikael Henaff, Joan Bruna, and Yann LeCun.
\newblock Deep convolutional networks on graph-structured data.
\newblock {\em arXiv:1506.05163}, 2015.

\bibitem{huang2018adaptive}
Wenbing Huang, Tong Zhang, Yu~Rong, and Junzhou Huang.
\newblock Adaptive sampling towards fast graph representation learning.
\newblock In {\em Advances in neural information processing systems}, pages
  4558--4567, 2018.

\bibitem{adam}
Diederik~P. Kingma and Jimmy~Lei Ba.
\newblock Adam: A method for stochastic optimization.
\newblock In {\em International Conference on Learning Representaion}, 2015.

\bibitem{kipf2016semi}
Thomas~N Kipf and Max Welling.
\newblock Semi-supervised classification with graph convolutional networks.
\newblock In {\em International Conference on Learning Representaion}, 2017.

\bibitem{levie2018cayleynets}
Ron Levie, Federico Monti, Xavier Bresson, and Michael~M Bronstein.
\newblock Cayleynets: Graph convolutional neural networks with complex rational
  spectral filters.
\newblock {\em IEEE Transactions on Signal Processing}, 67(1):97--109, 2018.

\bibitem{li2018deeper}
Qimai Li, Zhichao Han, and Xiao-Ming Wu.
\newblock Deeper insights into graph convolutional networks for semi-supervised
  learning.
\newblock In {\em Thirty-Second AAAI Conference on Artificial Intelligence},
  2018.

\bibitem{li2018adaptive}
Ruoyu Li, Sheng Wang, Feiyun Zhu, and Junzhou Huang.
\newblock Adaptive graph convolutional neural networks.
\newblock In {\em Thirty-second AAAI conference on artificial intelligence},
  2018.

\bibitem{liao2019lanczosnet}
Renjie Liao, Zhizhen Zhao, Raquel Urtasun, and Richard~S Zemel.
\newblock Lanczosnet: Multi-scale deep graph convolutional networks.
\newblock {\em arXiv:1901.01484}, 2019.

\bibitem{lovasz1993random}
L{\'a}szl{\'o} Lov{\'a}sz et~al.
\newblock Random walks on graphs: A survey.
\newblock {\em Combinatorics, Paul erdos is eighty}, 2(1):1--46, 1993.

\bibitem{mcauley2012image}
Julian McAuley and Jure Leskovec.
\newblock Image labeling on a network: using social-network metadata for image
  classification.
\newblock In {\em European conference on computer vision}, pages 828--841.
  Springer, 2012.

\bibitem{micheli2009neural}
Alessio Micheli.
\newblock Neural network for graphs: A contextual constructive approach.
\newblock {\em IEEE Transactions on Neural Networks}, 20(3):498--511, 2009.

\bibitem{monti2017geometric}
Federico Monti, Davide Boscaini, Jonathan Masci, Emanuele Rodola, Jan Svoboda,
  and Michael~M Bronstein.
\newblock Geometric deep learning on graphs and manifolds using mixture model
  cnns.
\newblock In {\em Proceedings of the IEEE Conference on Computer Vision and
  Pattern Recognition}, pages 5115--5124, 2017.

\bibitem{relu}
Vinod Nair and Geoffrey~E Hinton.
\newblock Rectified linear units improve restricted boltzmann machines.
\newblock In {\em Proceedings of the 27th international conference on machine
  learning (ICML-10)}, pages 807--814, 2010.

\bibitem{ren2020hgat}
Yuxiang Ren and Jiawei Zhang.
\newblock Hgat: Hierarchical graph attention network for fake news detection.
\newblock {\em arXiv preprint arXiv:2002.04397}, 2020.

\bibitem{rong2019dropedge}
Yu~Rong, Wenbing Huang, Tingyang Xu, and Junzhou Huang.
\newblock Dropedge: Towards deep graph convolutional networks on node
  classification.
\newblock {\em arXiv:1907.10903}, 2019.

\bibitem{scarselli2008graph}
Franco Scarselli, Marco Gori, Ah~Chung Tsoi, Markus Hagenbuchner, and Gabriele
  Monfardini.
\newblock The graph neural network model.
\newblock {\em IEEE Transactions on Neural Networks}, 20(1):61--80, 2008.

\bibitem{sen2008collective}
Prithviraj Sen, Galileo Namata, Mustafa Bilgic, Lise Getoor, Brian Galligher,
  and Tina Eliassi-Rad.
\newblock Collective classification in network data.
\newblock {\em AI magazine}, 29(3):93--93, 2008.

\bibitem{sperduti1997supervised}
Alessandro Sperduti and Antonina Starita.
\newblock Supervised neural networks for the classification of structures.
\newblock {\em IEEE Transactions on Neural Networks}, 8(3):714--735, 1997.

\bibitem{gat}
Petar Veli{\v{c}}kovi{\'{c}}, Guillem Cucurull, Arantxa Casanova, Adriana
  Romero, Pietro Li{\`{o}}, and Yoshua Bengio.
\newblock Graph attention networks.
\newblock In {\em International Conference on Learning Representaion}, 2018.

\bibitem{WMWG17}
Quan Wang, Zhendong Mao, Bin Wang, and Li~Guo.
\newblock Knowledge graph embedding: A survey of approaches and applications.
\newblock {\em IEEE Transactions on Knowledge and Data Engineering},
  29(12):2724--2743, 2017.

\bibitem{XHLJ10}
Keyulu Xu, Weihua Hu, Jure Leskovec, and Stefanie Jegelka.
\newblock How powerful are graph neural networks?
\newblock {\em arXiv:1810.00826}, 2018.

\bibitem{xu2018representation}
Keyulu Xu, Chengtao Li, Yonglong Tian, Tomohiro Sonobe, Ken-ichi Kawarabayashi,
  and Stefanie Jegelka.
\newblock Representation learning on graphs with jumping knowledge networks.
\newblock In {\em Proceedings of the 35th International Conference on Machine
  Learning}, 2018.

\bibitem{zeng2019graphsaint}
Hanqing Zeng, Hongkuan Zhou, Ajitesh Srivastava, Rajgopal Kannan, and Viktor
  Prasanna.
\newblock Graphsaint: Graph sampling based inductive learning method.
\newblock {\em arXiv preprint arXiv:1907.04931}, 2019.

\bibitem{gresnet}
Jiawei Zhang and Lin Meng.
\newblock Gresnet: Graph residual network for reviving deep gnns from suspended
  animation.
\newblock In {\em arXiv:1909.05729}, 2019.

\bibitem{zhao2019pairnorm}
Lingxiao Zhao and Leman Akoglu.
\newblock Pairnorm: Tackling oversmoothing in gnns.
\newblock {\em arXiv:1909.12223}, 2019.

\end{thebibliography}

\end{document}